\documentclass{article}
\usepackage{nips14submit_e,times}

\usepackage{amsmath}
\usepackage[numbers]{natbib}
\usepackage{algorithm}
\usepackage{amsthm}
\usepackage{amssymb}
\usepackage{bm}
\usepackage{bbm}
\usepackage{xspace}
\usepackage{dsfont}

\allowdisplaybreaks

\usepackage[disable]{todonotes}

\definecolor{PalePurp}{rgb}{0.66,0.57,0.66}

\newcommand{\transpose}{^\mathsf{\scriptscriptstyle T}}
\newcommand{\loss}{\ell}
\newcommand{\hloss}{\hat{\loss}}

\newcommand{\Sw}{\mathcal{S}}
\newcommand{\tSw}{\wt{\Sw}}
\newcommand{\bSw}{\bar{\Sw}}

\newcommand{\II}[1]{\mathds{1}_{\left\{#1\right\}}}
\newcommand{\PP}[1]{\mathbb{P}\left[#1\right]}
\newcommand{\EE}[1]{\mathbb{E}\left[#1\right]}

\newcommand{\EEc}[2]{\mathbb{E}\left[#1\left|#2\right.\right]}
\newcommand{\EEcc}[2]{\mathbb{E}\left[\left.#1\right|#2\right]}

\newcommand{\ev}[1]{\left\{#1\right\}}
\newcommand{\pa}[1]{\left(#1\right)}
\newcommand{\F}{\mathcal{F}}
\newcommand{\D}{\mathcal{D}}
\renewcommand{\P}{\mathcal{P}}
\newcommand{\OO}{\mathcal{O}}

\newcommand{\norm}[1]{\left\|#1\right\|}
\newcommand{\onenorm}[1]{\norm{#1}_1}

\def\argmin{\mathop{\rm arg\, min}}

\newcommand{\ra}{\rightarrow}

\newcommand{\bV}{\boldsymbol{V}}

\newcommand{\bv}{\boldsymbol{v}}

\newcommand{\tV}{\widetilde{\bV}}

\newcommand{\bloss}{\bm\ell}

\newcommand{\hbl}{\hat{\bloss}}
\newcommand{\bL}{\boldsymbol{L}}
\newcommand{\hbL}{\widehat{\bL}}
\newcommand{\bpi}{\bm{\pi}}

\newcommand{\bZ}{\boldsymbol{Z}}
\newcommand{\tbZ}{\widetilde{\bZ}}
\newcommand{\tbV}{\widetilde{\bV}}

\newcommand{\wt}{\widetilde}

\newcommand{\expn}{{\sc Exp4}\xspace}

\newcommand{\fpl}{{\sc FPL}\xspace}

\newtheorem{theorem}{Theorem}
\newtheorem{corollary}{Corollary}
\newtheorem{lemma}{Lemma}

\usepackage{svg}
\newcommand{\BSFPL}{{\sc BSFPL}\xspace}
\newcommand{\SleepingCat}{{\sc SleepingCat}\xspace}
\newcommand{\SleepingCatBandit}{{\sc SleepingCatBandit}\xspace}

\newcommand{\CombinatorialBSFPL}{{\sc CombBSFPL}\xspace}

\begin{document}
\nipsfinalcopy

\author{
Gergely Neu \qquad Michal Valko\\
SequeL team, INRIA Lille -- Nord Europe, France\\
\texttt{\small \{gergely.neu,michal.valko\}@inria.fr} \\
}

 \title{Online combinatorial optimization with
stochastic decision sets and adversarial losses}

 \maketitle

\newcommand\blfootnote[1]{%
  \begingroup
  \renewcommand\thefootnote{}\footnote{#1}%
  \addtocounter{footnote}{-1}%
  \endgroup
}

\begin{abstract}
Most work on sequential learning assumes a fixed set of actions that are
available all the time to choose from. However, in practice, actions can
consist of picking subsets of readings from sensors that may break from time to time, 
road segments that can be blocked or goods that are out of stock. In this paper we study
learning algorithms that are able to deal with \textit{stochastic
availability} of such unreliable composite actions. We propose and analyze algorithms based on the
Follow-The-Perturbed-Leader prediction method for several learning settings differing in the feedback provided to
the learner. Our techniques rely on a novel loss estimation technique that we call \textit{Counting Awake Times}. We
deliver improved regret bounds for our algorithms for the previously studied \textit{full information} and
\textit{bandit} settings. Finally, we evaluate our algorithms and show their improvement over
the known approaches.
\end{abstract}

\section{Introduction}
In online learning problems~\cite{CBLu06:Book} we aim to sequentially
select actions from a given set in order to optimize some performance measure. However,
in many
sequential learning problems we have to deal with situations when some of 
the actions are not available to be taken.
A simple and well-studied problem where such situations arise is that of sequential routing \citep{gyorgy07sp}, where we
have to select every day an itinerary for commuting from home to work so as to minimize the total time spent driving (or
even worse, stuck in a traffic jam). In this scenario, some road segments may be blocked for maintenance, forcing us to
work with the rest of the road network. This problem is isomorphic to packet routing in ad-hoc computer networks where
some links might not be always available because of a faulty transmitter or a depleted battery.
Another important class of sequential decision-making problems where the decision space might change over time
is recommender systems~\cite{jannach2010recommender}. Here, some items may be out of stock or some service
may not be applicable at some time (e.g., a movie not shown that day, bandwidth
issues in video streaming services).
In these cases, the advertiser may refrain from recommending unavailable items. Other reasons include a distributor
being overloaded with commands or facing shipment problems.

Learning problems with such partial-availability restrictions have been previously studied in the framework of
prediction with expert advice. \citet{FSSW97} considered the problem of online
prediction with \emph{specialist experts}, where some experts' predictions might not be available from time to time,
and the goal of the learner is to minimize regret against the best \emph{mixture} of experts. \citet{KNMS08} proposed a
stronger notion of regret measured against the best \emph{ranking} of experts and gave efficient algorithms that work
under stochastic assumptions on the losses, referring to this setting as prediction with \textit{sleeping experts}. They
have also introduced the notion of \emph{sleeping bandit} problems
where the learner only gets partial feedback about its decisions. 
They gave an inefficient algorithm for the non-stochastic case, with some
hints that it might be difficult to learn efficiently in this general setting. This was later reaffirmed by 
\citet{KS12}, who reduce the problem of PAC learning of DNF formulas to a non-stochastic sleeping experts problem,
proving the hardness of learning in this setup.
Despite these negative results, \citet{kanade09sleeping} have shown that there is still hope to obtain efficient
algorithms in adversarial environments, if one introduces a certain \emph{stochastic assumption on the
decision set}. 

In this paper, we extend the work of \citet{kanade09sleeping} to combinatorial settings where the action set of the
learner is possibly huge, but has a compact representation. We also assume \textit{stochastic action availability}: in
each decision period, the decision space is drawn from a \textit{fixed but unknown} probability distribution
independently of the history of interaction between the learner and the
environment. The goal of the learner is to
minimize the sum of losses associated with its decisions. As usual in online
settings, we measure the performance of
the learning algorithm by its regret defined as the gap between the total loss
of the best fixed decision-making
policy from a pool of policies and the total loss of the learner. The choice of this pool, however, is a rather
delicate question in our problem: the usual choice of measuring regret against the best fixed action is meaningless,
since not all actions are available in all time steps. Following \citet{kanade09sleeping} (see also \cite{KNMS08}), we
consider the policy space composed of all mappings from decision sets to actions within the respective sets.

We study the above online combinatorial optimization setting under three feedback assumptions. 
Besides the full-information and bandit settings considered by \citet{kanade09sleeping}, we also consider
a restricted feedback scheme as a natural middle ground between the two by assuming that the learner gets to know the
losses associated only with \emph{available} actions. This extension (also studied by \cite{KNMS08}) is crucially
important in practice, since in most cases it is unrealistic to expect that an
unavailable expert would report its loss. Finally, we also consider a generalization of bandit feedback to the
combinatorial case known as \emph{semi-bandit} feedback.
Our main contributions in this paper are two algorithms called \SleepingCat and \SleepingCatBandit that work in the
restricted and semi-bandit information schemes, respectively.
The best known competitor of our algorithms is the \BSFPL algorithm
of~\citet{kanade09sleeping} that works in two
phases. First, an initial phase is dedicated to the estimation of the 
distribution of the available actions.
Then, in the main phase, \BSFPL randomly alternates between exploration and
exploitation. Our technique improves over the
\fpl-based method of \citet{kanade09sleeping} by
removing the costly exploration phase dedicated to estimate the availability
probabilities, and also the explicit exploration steps in their main phase. 
This is achieved by a cheap alternative loss
estimation procedure called Counting Asleep Times (or CAT)
that does not require estimating the distribution of the action sets. This technique improves 
the regret bound of \cite{kanade09sleeping} after $T$ steps from $\OO(T^{4/5})$
to $\OO(T^{2/3})$ in their setting, and also provides a
regret guarantee of $\OO(\sqrt{T})$ in the restricted setting.\footnote{While 
not explicitly proved by
\citet{kanade09sleeping}, their technique can be extended to work in the restricted setting, where it can be
shown to guarantee a regret of $\OO(T^{3/4})$.}

\section{Background}
We now give the formal definition of the learning problem.
\begin{figure*}[t]
\centering
\fbox{
\begin{minipage}{.95\textwidth}
{\bfseries Parameters}: \\
\phantom{aa}full set of decision vectors $\Sw = \ev{0,1}^d$,
number of rounds $T$, unknown distribution $\P\in\Delta_{2^\Sw}$ \\
{\bfseries For all $t=1,2,\dots,T$ repeat}
\begin{enumerate}
 \item The environment draws a set of available actions $\Sw_t\sim\P$ and
picks a loss vector
$\bloss_t\in[0,1]^d$.
 \item The set $\Sw_t$ is revealed to the learner.
 \item Based on its previous observations (and possibly some source of
randomness), the learner picks an action
$\bV_t\in\Sw_t$.
 \item The learner suffers loss $\bV_t^\top\bloss_{t}$ and gets some feedback:
 \begin{enumerate}
  \item in the \emph{full information} setting, the learner observes
$\bloss_{t}$,
  \item in the \emph{restricted} setting, the learner observes $\loss_{t,i}$ for
all $i\in\D_t$,
  \item in the \emph{semi-bandit} setting, the learner observes $\loss_{t,i}$
for all $i$ such that $V_{t,i}=1$.
  \end{enumerate}
\end{enumerate}
\end{minipage}
}
\caption{The protocol of online combinatorial optimization with stochastic
action availability.}
\label{fig:protocol}
\end{figure*}
We consider a sequential interaction scheme between a learner and an environment
where in each round $t\in[T]=\ev{1,2,\dots,T}$,
the learner has to choose an action $\bV_t$ from a subset $\Sw_t$ of a known
decision set $\Sw\subseteq\ev{0,1}^d$ with $\onenorm{\bv}\le m$ for all $\bv\in\Sw$.  We
assume that the environment selects $\Sw_t$ according to some fixed (but
unknown) distribution $\P$, independently of the
interaction history. Unaware of the learner's decision, the environment also
decides on a loss vector
$\bloss_t\in[0,1]^d$ that will determine the loss suffered by the learner, which
is of the form $\bV_t^\top \bloss_t$.
We make no assumptions on how the environment generates the sequence of loss
vectors, that is, we are interested in
algorithms that work in non-oblivious (or adaptive) environments. At the end of
each round, the learner receives some
feedback based on the loss vector and the action of the learner. The goal of the
learner is to pick its actions so as to
minimize the losses it accumulates by the end of the $T$'th round. This setup generalizes the setting of online
combinatorial optimization considered by \citet{CL12,audibert13regret}, where the decision set is assumed to be fixed
throughout the learning procedure.
The interaction protocol is summarized on
Figure~\ref{fig:protocol} for reference.

We distinguish between three different feedback schemes, the
simplest one being the \emph{full
information} scheme
where the loss vectors are completely revealed to the learner at the end of each
round. In the
\emph{restricted-information} scheme, we make a much milder assumption that the
learner is
informed about the
losses of the
\emph{available} actions. Precisely, we define the set of \emph{available
components} as
\[
 \D_t = \ev{i\in[d]:\exists \bv\in\Sw_t: v_i=1}
 \]
 and assume that the learner can observe the $i$-th component of the loss vector
$\bloss_t$ if and only if $i\in \D_t$.
This is a sensible assumption in a number of practical applications, e.g., in
sequential routing problems where components are associated with links in a network. Finally, in the \emph{semi-bandit}
scheme, we
assume that the learner only
observes losses associated with the components of its own decision, that is, the
feedback is $\loss_{t,i}$ for all $i$
such that $V_{t,i}=1$. This is the case in online advertising
settings where components of the decision
vectors represent customer-ad allocations. The
observation history  $\F_{t}$ is defined
as the sigma-algebra generated by the actions chosen by the learner and the
decision sets handed out by the environment
by the end of round $t$: $\F_t = \sigma(\bV_t,\Sw_t,\dots,\bV_1,\Sw_1)$.

The performance of the learner is measured with respect to the best fixed
\emph{policy} (otherwise known as a \emph{choice function} in discrete choice
theory~\cite{koshevoy1999choice}) of the form
$\bpi:2^\Sw\ra \Sw$.
In words, a policy $\bpi$ will pick action $\bpi(\bSw)\in \bSw$ whenever the
environment selects action set $\bSw$. The
(total expected) \emph{regret} of the learner is defined as
\begin{equation}\label{eq:regret}
 R_T = \max_{\bpi} \sum_{t=1}^T \EE{\left(\bV_t - \bpi(\Sw_t)\right)^\top
\bloss_t}.
\end{equation}
Note that the above expectation integrates over \emph{both} the randomness
injected by the learner \emph{and} the
stochastic process generating the decision sets. The attentive reader might
notice that this regret criterion
is very similar to that of \citet{kanade09sleeping}, who study the setting of
prediction with expert advice (where
$m=1$) and measure regret against the best fixed \emph{ranking}
of experts.

\section{Loss estimation by Counting Asleep Times}\label{sec:cat}
 In this section, we describe our method used for estimating unobserved losses that works \emph{without having to
explicitly learn
the availability distribution $\P$}. To explain the concept on a high level, let
us now consider our simpler partial-observability setting, the
restricted-information setting.
For the formal treatment of the problem, let us fix any component $i\in[d]$ and define $A_{t,i} = \II{i\in \D_t}$
and $a_i = \EEc{A_{t,i}}{\F_{t-1}}$. Had we known the observation probability $a_i$, we would be able to estimate the
$i$'th component of the loss vector $\bloss_{t}$ by $\hloss_{t,i}^* = (\loss_{t,i}A_{t,i}) / a_i$,
as the quantity $\loss_{t,i}A_{t,i}$ is observable. It is easy to see that 
the estimate  $\hloss_{t,i}^*$ is unbiased by definition --
but, unfortunately, we do not know $a_i$, so we have no hope to compute it. A simple idea used by
\citet{kanade09sleeping} is to devote the first $T_0$ rounds of interaction solely to the purpose of estimating $a_i$
by the sample mean $\hat{a}_i = (\sum_{t=1}^{T_0} A_{t,i})/T_0$. While this trick gets the job done, it is obviously
wasteful as we have to throw away all loss observations before the estimates 
are sufficiently concentrated.
\footnote{Notice that we require ``sufficient concentration'' from $1/\hat{a}_i$
and not only from $\hat{a}_i$! The deviation
of such quantities is rather difficult to control, as demonstrated by the complicated analysis of
\citet{kanade09sleeping}.}

We take a much simpler approach based on the observation that the 
``asleep-time'' of component $i$
is a geometrically distributed random variable with parameter $a_i$. The asleep-time of component $i$ starting from time
$t$ is formally defined as
\[
 N_{t,i} = \min \ev{n>0: i\in\D_{t+n}},
\]
which is the number of rounds until the next observation of the loss associated with component $i$.
Using the above definition, we construct our loss estimates as the vector $\hbl_t$ whose $i$-th component is
\begin{equation}\label{eq:catest}
 \hloss_{t,i} = \loss_{t,i} A_{t,i} N_{t,i}.
\end{equation}
It is easy to see that the above loss estimates are unbiased as 
\[
\begin{split}
\EEc{\loss_{t,i} A_{t,i} N_{t,i}}{\F_{t-1}} = 
\loss_{t,i}\EEc{A_{t,i}}{\F_{t-1}}\EEc{N_{t,i}}{\F_{t-1}}
= \loss_{t,i} a_i \cdot \frac{1}{a_i} = \loss_{t,i}
\end{split}
\]
for any $i$. We will refer to this loss-estimation method as \emph{Counting Asleep Times} (CAT). 
Looking at the definition~\eqref{eq:catest}, the attentive reader might worry
that the vector $\hbl_t$ \emph{depends on
future realizations of the random decision sets} and thus could be useless for 
practical use. However, observe that
there is no reason that the learner should use the estimate $\hloss_{t,i}$ before component $i$ wakes up in round
$t+N_{t,i}$ -- which is precisely the time when the
estimate becomes well-defined. This suggests a very simple implementation of CAT: whenever a component is not available,
estimate its loss by the last observation from that component! More formally, set
\[
 \hloss_{t,i} = 
 \begin{cases}
  \loss_{t,i}, &\mbox{if $i\in\D_t$}\\
  \hloss_{t-1,i}, &\mbox{otherwise.}
 \end{cases}
\]
It is easy to see that at the beginning of any round $t$, the two alternative definitions match for all components
$i\in\D_t$. In the next section, we confirm that this property is sufficient for running our algorithm.

\section{Algorithms \& their analyses}
For all information settings, we base our learning algorithms on the Follow-the-Perturbed-Leader (\fpl) prediction
method of \citet{Han57}, as popularized by \citet{KV05}. This algorithm works by additively perturbing the total
estimated loss of each
component, and then running an optimization oracle over the perturbed losses to choose the next action. More precisely,
our algorithms maintain the cumulative sum of their loss estimates $\hbL_{t} = \sum_{s=1}^t \hbl_t$ and pick the
action
\[
 \bV_t = \argmin_{\bv\in\Sw_t} \bv\transpose \pa{\eta \hbL_{t-1} - \bZ_t},
\]
where $\bZ_t$ is a perturbation vector with independent exponentially distributed components with unit expectation,
generated independently of the history, and $\eta>0$ is a parameter of the algorithm. Our algorithms for the different
information settings will be instances of
FPL that employ different loss estimates suitable for the respective settings. In the first part of this section, we
present the main tools of analysis that will be used for each
resulting method.

As usual for analyzing \fpl-based methods \citep{KV05,HuPo04,NeuBartok13}, we start by defining a hypothetical
forecaster that uses a time-independent perturbation vector $\tbZ$ with standard exponential components and
peeks one step into the future. However, we need an extra trick to deal with the
randomness of the decision set: we introduce the time-independent
decision set $\tSw\sim\P$ (drawn independently of the filtration $(\F_t)_t$) and
define
\[
 \tbV_t = \argmin_{\bv\in\tSw} \bv\transpose \pa{\eta\hbL_{t} - \tbZ}.
\]
Clearly, this forecaster is infeasible as it uses observations from the future. Also observe that
$\tbV_{t-1}\sim\bV_t$ given $\F_{t-1}$. The following two lemmas show how analyzing this forecaster can help in
establishing the performance of our actual algorithms.
\begin{lemma}\label{lem:cheat}
For any sequence of loss estimates, the expected regret of the hypothetical forecaster against any fixed policy
$\bpi:2^\Sw\ra\Sw$ satisfies
 \[
\EE{\sum_{t=1}^T \left(\tV_t - \bpi(\tSw)\right)\transpose \hbl_t}  \le \frac{m\left(\log d+1\right)}{\eta}.
 \]
\end{lemma}
The statement is easily proved by applying the
follow-the-leader/be-the-leader lemma\footnote{This lemma can be proved in the 
current case by virtue of the fixed
decision set $\tSw$, allowing the necessary recursion steps to go through.} (see,
e.g., \citep[Lemma~3.1]{CBLu06:Book}) to the loss sequence
$(\hbl_1-\tbZ,\hbl_2,\dots,\hbl_T)$, reorganizing, and using the upper bound
$\EE{\bigl\|\tbZ\bigr\|_\infty}\le \log d+1$. 

The following result can be extracted from the proof of Theorem~1 of \citet{NeuBartok13}.
\begin{lemma}\label{lem:price}
For any sequence of nonnegative loss estimates,
\[
 \EEcc{(\tbV_{t-1} - \tV_t)\transpose \hbl_t}{\F_{t-1}} \le \eta\, 
\EEcc{\left(\tV_{t-1}\transpose \hbl_{t}\right)^2}{\F_{t-1}}.
\] 
\end{lemma}
In the next subsections, we apply these results to obtain bounds for the three information settings.

\subsection{Algorithm for full information}
In the simplest setting, we can use $\hbl_t = \bloss_t$, which yields the 
following theorem:
\begin{theorem}\label{thm:fullinfo}
Define
\[
 L_T^* = \max\ev{\min_\pi \EE{\sum_{t=1}^T \pi(\Sw_t)\transpose \bloss_t}, 
4(\log d + 1)}.
\]
Setting $\eta = \sqrt{(\log d + 1)/L_T^*}$, the regret of FPL in the full information scheme satisfies
 \[
  R_T \le 2 m\sqrt{2 L_T^* (\log d + 1)}.
 \]
\end{theorem}
As this result is comparable to the best available bounds for FPL \cite{HuPo04,NeuBartok13} in the full information
setting and a \emph{fixed} decision set, it reinforces the observation of \citet{kanade09sleeping}, who show that the
sleeping experts problem with full information and stochastic availability is no more difficult than the standard
experts problem. 
The proof of Theorem~\ref{thm:fullinfo} follows directly from combining 
Lemmas~\ref{lem:cheat} and~\ref{lem:price} with
some standard tricks. For completeness, details are provided in Appendix~\ref{app:fullinfo}.

\subsection{Algorithm for restricted feedback}
In this section, we use the CAT loss estimate defined in
Equation~\ref{eq:catest} as $\hbl_t$ in FPL, and call the
resulting method \SleepingCat.
The following theorem gives the performance guarantee for this algorithm.
\begin{theorem}\label{thm:restricted}
Define $Q_t = \sum_{i=1}^d \EEcc{V_{t,i}}{i\in \D_t}$.
 The total expected regret of \SleepingCat against the best fixed policy is
upper bounded
as 
\[
 R_T \le \frac{m(\log d + 1)}{\eta} + 2\eta m\sum_{t=1}^T Q_t.
\]
\end{theorem}

\begin{proof}[Proof sketch of Theorem~\ref{thm:restricted}]
We start by observing $\EE{\bpi(\tSw)\transpose\hbl_t} =
\EE{\bpi(\Sw_t)\transpose\bloss_t}$, where we used that
$\hbl_t$ is independent of $\tSw$ and is an unbiased estimate of $\bloss_t$, and
also that $\Sw_t\sim\tSw$.
The proof is completed by combining this with Lemmas~\ref{lem:cheat}
and~\ref{lem:price}, and the bound
\[
 \EEcc{\left(\tV_{t-1}\transpose \hbl_{t}\right)^2}{\F_{t-1}} \le 2m Q_t.
\]
The proof of this last statement follows from a tedious calculation that we
defer to Appendix~\ref{app:restricted}.
\end{proof}
Below, we provide two ways of further bounding the regret under various
assumptions. The first one provides a
universal upper bound that holds without any further assumptions.
\begin{corollary}\label{cor:rest_1}
 Setting $\eta=\sqrt{(\log d + 1)/(2dT)}$, the regret of \SleepingCat against
the best fixed policy is bounded as
 \[
  R_T \le 2m\sqrt{2dT(\log d + 1)}.
 \]
\end{corollary}
The proof follows from the fact that $Q_t\le d$ no matter what $\P$ is. A somewhat surprising feature of our bound is
its scaling with $\sqrt{d\log d}$, which is much worse than the logarithmic dependence exhibited in the full
information case. It is easy to see, however, that this bound is not improvable in general -- see
Appendix~\ref{app:lowerbound} for a simple example.
The next bound, however, shows that it is possible to improve this bound by assuming that most components are
reliable in some sense, which is the case in many practical settings.
\begin{corollary}\label{cor:rest_2}
 Assuming $a_i\ge \beta$ for all $i$, we have $Q_t \le 1/\beta$, and
 setting $\eta=\sqrt{\beta(\log d + 1)/(2T)}$ guarantees that the regret of \SleepingCat against the best fixed policy
is bounded as
 \[
  R_T \le 2m\sqrt{\frac{2T(\log d + 1)}{\beta}}.
 \]
\end{corollary}

\subsection{Algorithm for semi-bandit feedback}\label{sec:semibandit}
In this section, we turn our attention to the problem of learning with semi-bandit feedback where the learner only gets
to observe the losses associated with its own decision. Specifically, we assume that the learner observes all
components $i$ of the loss vector such that $V_{t,i}=1$. The extra difficulty in this setting is that our actions
influence the feedback that we receive, so we have to be more careful when defining our loss estimates. Ideally, we
would like to work with unbiased estimates of the form
\begin{equation}\label{eq:est_ideal}
 \hloss^*_{t,i} = \frac{\loss_{t,i}}{q_{t,i}^*}V_{t,i}, \qquad\mbox{where}\qquad q^*_{t,i} = \EEcc{V_{t,i}}{\F_{t-1}} =
\sum_{\bSw\in 2^\Sw} \P(\bSw)\EEc{V_{t,i}}{\F_{t-1},\Sw_t=\bSw}.
\end{equation}
for all $i\in[d]$.
Unfortunately though, we are in no position to compute these estimates, as this would require perfect knowledge of the
availability distribution $\P$! Thus we have to
look for another way to compute reliable loss estimates. A possible idea is to use
\[
q_{t,i} \cdot a_i = \EEcc{V_{t,i}}{\F_{t-1},\Sw_t} \cdot \PP{i\in D_t}.
\]
instead of $q_{t,i}^*$ in Equation~\ref{eq:est_ideal} to normalize the observed
losses. This choice yields another
unbiased loss estimate as
\begin{equation}\label{eq:est_almost}
\begin{split}
 \EEcc{\frac{\loss_{t,i}V_{t,i}}{q_{t,i}a_i}}{\F_{t-1}}
 &= \frac{\loss_{t,i}}{a_i} \EEcc{\EEcc{\frac{V_{t,i}}{q_{t,i}}}{\F_{t-1},\Sw_t}}{\F_{t-1}}
 = \frac{\loss_{t,i}}{a_i} \EEcc{A_{t,i}}{\F_{t-1}} = \loss_{t,i},
\end{split}
\end{equation}
which leaves us with the problem of computing $q_{t,i}$ and $a_i$. While this also seems to be a tough challenge, we
now show to estimate this quantity by generalizing the CAT technique presented in Section~\ref{sec:cat}.

Besides our trick used for estimating the $1/a_i$'s by the random variables $N_{t,i}$, we now also have to face the
problem of not being able to find a closed-form expression for the $q_{t,i}$'s.  
Hence, we follow the geometric
resampling approach of \citet{NeuBartok13} and draw an additional sequence of $M$ perturbation vectors
$\bZ_t'(1),\dots,\bZ_t'(M)$ and use them to define
\[
 \bV_t'(k) = \argmin_{\bv\in\Sw_t} \ev{\hbL_{t-1} - \bZ_t'(k)}
\]
for all $k\in[M]$. Using these simulated predictions, we define
\[
 K_{t,i} = \min\pa{\ev{k\in\left[M\right]: V_{t,i}'(k) = V_{t,i}} \cup \ev{M}}.
\]
and
\begin{equation}
 \hloss_{t,i} = \loss_{t,i} K_{t,i} N_{t,i} V_{t,i}
\end{equation}
for all $i$. Setting $M=\infty$ makes this expression equivalent to $\frac{\loss_{t,i}V_{t,i}}{q_{t,i}a_i}$ in
expectation, yielding yet another unbiased estimator for the losses. Our analysis, however,
crucially relies on setting $M$ to a finite value so as to control the variance of the loss estimates. We are not aware
of any other work that achieves a similar variance-reduction effect without explicitly exploring the
action space \cite{McMaBlu04,DaHaKa07,CL12,BCK12}, making this alternative bias-variance tradeoff a unique feature
of our analysis. We call the algorithm resulting from using the loss estimates above \SleepingCatBandit. The following
theorem gives the performance guarantee for this algorithm.

\begin{theorem}
 \label{thm:semibandit}
 Define $Q_t = \sum_{i=1}^d \EEcc{V_{t,i}}{i\in D_t}$.
 The total expected regret of \SleepingCatBandit against the best fixed policy
is bounded as
\[
 R_T \le \frac{m(\log d + 1)}{\eta} + 2\eta Mm\sum_{t=1}^T Q_t + \frac{dT}{eM}.
\]
\end{theorem}

\begin{proof}[Proof of Theorem~\ref{thm:semibandit}]
First, observe that $\EEcc{\hloss_{t,i}}{\F_{t-1}} \le \loss_{t,i}$ as
$\EEcc{K_{t,i}N_{t,i}}{\F_{t-1}}\le
1/(q_{t,i}a_i)$ by definition. Thus, we can get $\EE{\bpi(\tSw)\transpose\hbl_t} \le
\EE{\bpi(\Sw_t)\transpose\bloss_t}$
by a similar argument that we used in the proof of Theorem~\ref{thm:restricted}.
 After yet another long and tedious calculation (see
Appendix~\ref{app:semibandit}), we can prove
 \[
   \EEcc{\left(\tV_{t-1}\transpose \hbl_{t}\right)^2}{\F_{t-1}} \le 2Mm Q_t.
 \]
 The proof is concluded by combining this bound with Lemmas~\ref{lem:cheat}
and~\ref{lem:price} and the upper bound
 \[
  \EEcc{\bV_{t}\transpose\bloss_t}{\F_{t-1}} \le
\EEcc{\tbV_{t-1}\transpose\hbl_t}{\F_{t-1}} + \frac{d}{eM},
 \]
 which can be proved by directly following the techniques of
\citet{NeuBartok13}.
\end{proof}

\begin{corollary}
 Setting $
\eta=\pa{\frac{\sqrt{m}(\log d + 1)}{2dT}}^{2/3}
 $
 and
 $
  M = \frac {1}{\sqrt e} \cdot \pa{\frac{dT}{\sqrt{2} m (\log d + 1)}}^{1/3}
 $
 guarantees that the regret of \SleepingCatBandit against the best fixed policy
is bounded as
 \[
  R_T \le (2mdT)^{2/3}\cdot (\log d+1)^{1/3}.
 \]
\end{corollary}
The proof of the corollary follows from bounding $Q_t\le d$ and plugging
the parameters into the bound of Theorem~\ref{thm:semibandit}.

This corollary implies that \SleepingCatBandit achieves a regret of $(2KT)^{2/3}\cdot (\log K+1)^{1/3}$ in the case
when $\Sw=[K]$, that is, in the $K$-armed sleeping bandit problem considered by \citet{kanade09sleeping}. This improves
their bound of $O((KT)^{4/5}\log T)$ by a large margin, thanks to the fact that
we did not have to explicitly learn the
distribution $\P$. 

Still, one might ask if it is possible to achieve a regret of order $\sqrt{T}$ in the semi-bandit setting. While the
\expn algorithm of \citet{auer2002bandit} can be used to obtain such regret
guarantee, running this algorithm is out
of question as its time and space complexity can be double-exponential in $d$
(see also the comments in~\cite{KNMS08}). Had
we had access to the loss estimates~\eqref{eq:est_ideal}, we would be able to
control the regret of FPL as the term
on the right hand side of Lemma~\ref{lem:cheat} could be bounded by $md$, which is sufficient for obtaining a regret
bound of $O(m\sqrt{dT\log d})$. Unfortunately,  we cannot achieve similar results even with the
information-theoretically feasible estimate of Equation~\ref{eq:est_almost}.
This obstacle is present
in the simple multi-armed bandit case, too. We conclude that learning in the bandit setting requires significantly more
knowledge about $\P$ than the knowledge of the $a_i$'s. The question if we can extend the CAT technique to estimate all
the relevant quantities of $\P$ is an interesting problem left for future investigation.

Finally, we note that similarly to the improvement of Corollary~\ref{cor:rest_2}, it is possible to replace the factor
$d^{2/3}$ by $(d/\beta)^{1/3}$ if we assume that $a_i\ge \beta$ for all $i$ and some $\beta>0$.

\section{Experiments}

\begin{figure}
\begin{center}
\vspace{-1em}
\includegraphics[width=0.32\columnwidth]{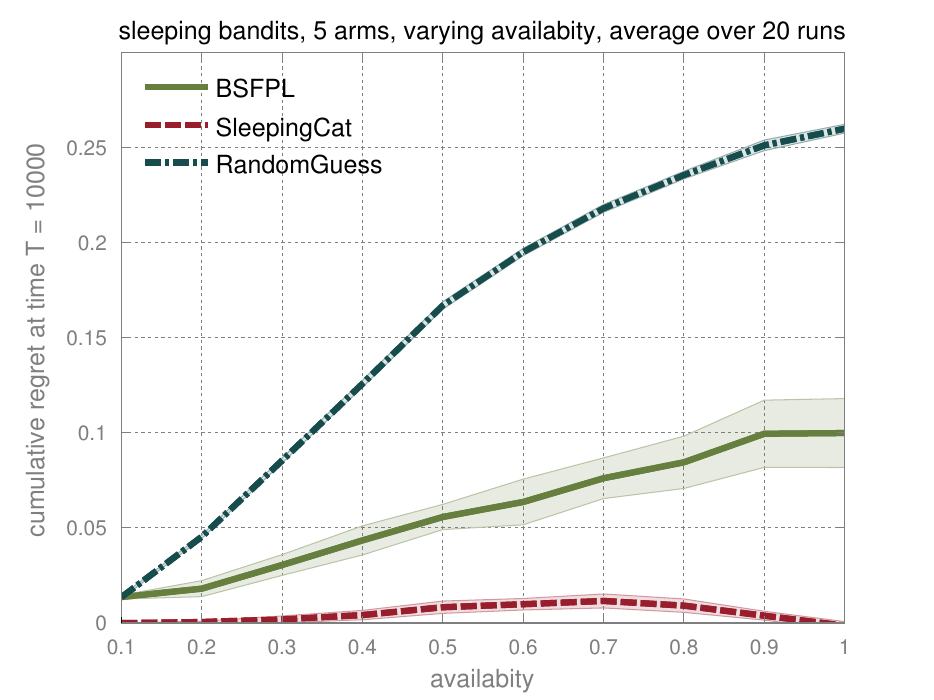}
\includegraphics[width=0.32\columnwidth]{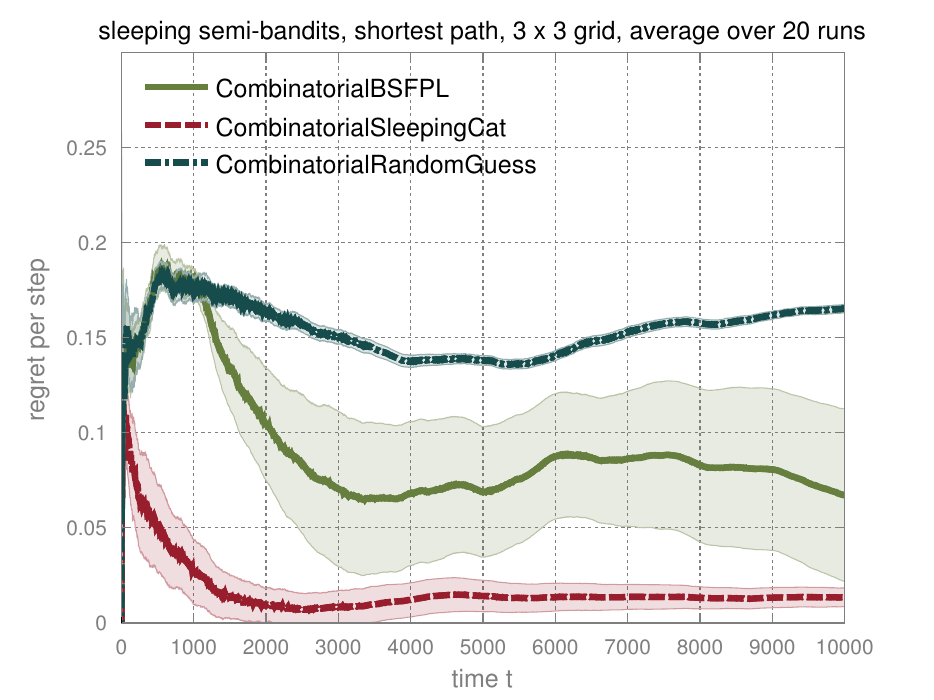}
\includegraphics[width=0.32\columnwidth]{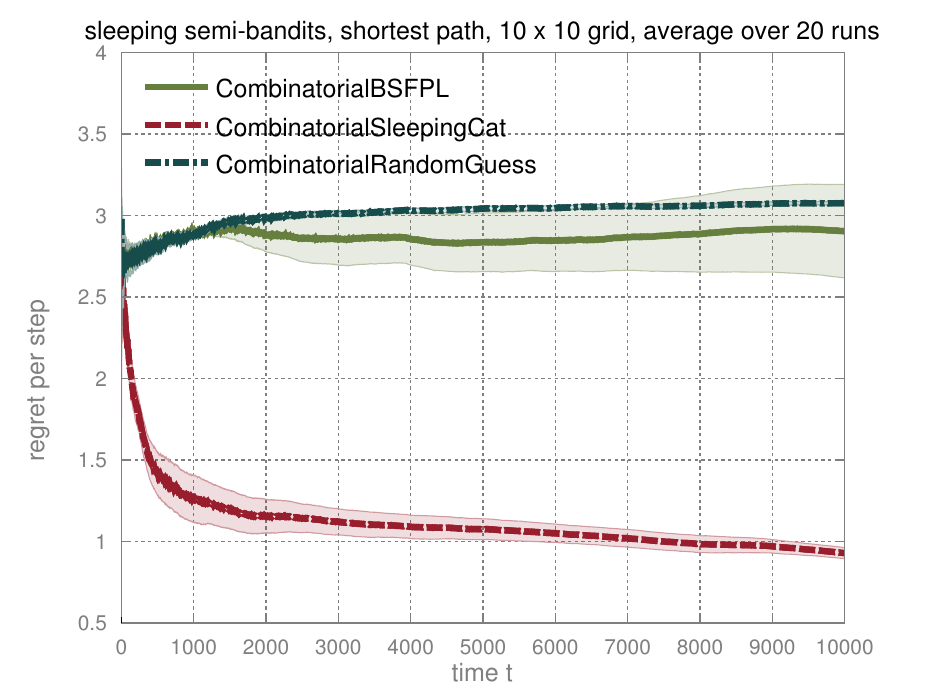}
\vspace{-1em}
\caption{\textbf{Left:} Multi-arm bandits - varying availabilities.
\textbf{Middle:} Shortest paths on a  $3\times3$ grid.
\textbf{Right:} Shortest paths on a  $10\times10$ grid.
}
\vspace{-1em}
\label{fig:exp}
 \end{center}
\end{figure}

In this section we present the empirical evaluation of our algorithms
for bandit and semi-bandit settings, and compare them
to their counterparts~\cite{kanade09sleeping}. We demonstrate
that the wasteful exploration of \BSFPL does not only result in worse regret
bounds but also \textit{degrades its empirical performance}.

For the bandit case, we evaluate \SleepingCatBandit using the same empirical
setting as Kanade~et~al.~\cite{kanade09sleeping}. We consider
an experiment with $T = 10^4$ and 5 arms, each of which are available independently of each other with probability $p$.
At the beginning of the experiment, the losses of each arm are
chosen uniformly at random from $[0,1]$. After that, they follow a random walk
defined in each step as a random additive perturbation of the previous step
by a sample from the zero-mean Gaussian with $\sigma=0.002$. 
Losses outside $[0,1]$ are truncated. 
In our first experiment (Figure~\ref{fig:exp}, left), we study the effect of changing $p$ on the performance of \BSFPL
and \SleepingCatBandit. Notice that when $p$ is very low, there are few or no arms to choose
from. In this case the problems are easy by design and all algorithms suffer low
regret. As $p$ increases, the policy space starts to blow up 
and the problem becomes more difficult.
When $p$ approaches one, the problem collapses into the set of single arms and the problem gets easier again. Observe that the
behavior of \SleepingCatBandit
follows this trend. On the other hand, the performance of \BSFPL  steadily
decreases with increasing availability.  This is due to the explicit
exploration rounds in the main phase of \BSFPL, that suffers the loss of the uniform policy scaled by the exploration
probability. The performance of the uniform policy is plotted for reference.

To evaluate \SleepingCatBandit in the semi-bandit setting, we
consider the \textit{shortest path problem} on grids of $3 \times 3$ and $10 \times 10$ nodes, which
amounts to 12 and 180 edges respectively. For each edge, we
generate a random-walk loss sequence in the same way as
in our first experiment. In each round $t$, the learner has to choose a path from the lower left corner to the upper
right one composed from available edges. The individual availability of each edge is sampled with probability
0.9, independent of the others. Whenever an edge gets disconnected from the source, it becomes unavailable itself,
resulting in a quite complex availability structure.
Once a learner chooses a path, the losses of chosen road segments are
revealed and the learner suffers their sum. As an offline oracle for the
shortest path problem we use Dijkstra's algorithm.

Since \cite{kanade09sleeping} does not provide a combinatorial version of their
approach, we compare against \CombinatorialBSFPL{}, a straightforward extension
of \BSFPL{}. As in \BSFPL{}, we dedicate the initial phase to estimate the
availabilities of the component. Notice that this is different from the
bandit case, as the edge may be available, but not reachable due to the
combinatorial constraint of the problem. We therefore estimate
the reachability of components using the oracle. In the main phase, we
follow \BSFPL{} and alternate between the exploration and exploitation. In the
exploration rounds, we test for the reachability of the
randomly sampled component and update the reward
estimates as in \BSFPL{}.

Figure~\ref{fig:exp} shows the performance
of \CombinatorialBSFPL{} and \SleepingCatBandit{} for a fixed loss sequence,
averaged over 20 samples of the component availabilities.
To appraise the benefit of learning in this problem, we
also plot the performance of random policy, which pulls some available arm
uniformly at random. We notice that the initial exploration phase of
\CombinatorialBSFPL{}  suffers a regret comparable to the random policy.
The drawback of \CombinatorialBSFPL{} is the explicit separation of the
exploration and the exploitation rounds. The explicit exploration of
\CombinatorialBSFPL{} only happens on the fraction
of the rounds while \SleepingCatBandit{} with its CAT trick explores implicitly
in each round.  This drawback  is far more apparent
when the number of components increases, as it is the case for the $10 \times
10$ grid graph with 180 components.

Our experiments give evidence that the implicit exploration of
\SleepingCatBandit{} results in the significantly better performance than the
approach of~\cite{kanade09sleeping}, which suffers both from the initial
exploration phase and from the explicit exploration rounds in the main
phase.

\paragraph{Acknowledgements}
\label{sec:Acknowledgements}
The research presented in this paper was supported by French Ministry of
Higher Education and Research, by European Community's
Seventh Framework Programme (FP7/2007-2013) under grant agreement n$^{\rm
o}$270327 (CompLACS), and by FUI project Herm\` es.

\vfil
\newpage
\bibliographystyle{apalike}
\bibliography{allbib,ngbib,shortconfs,dijkstra}

\newpage

\appendix 

\section{Proof of Theorem~\ref{thm:fullinfo}}\label{app:fullinfo}
We begin by applying Lemma~\ref{lem:cheat}, exploiting that $\hbl_t = \bloss_t$ and $\tSw$ is an independent copy of
$\Sw_t$:
\[
\begin{split}
\sum_{t=1}^T \EE{\tV_t\transpose\bloss_t} - \sum_{t=1}^T \EE{\bpi(\Sw_t)\transpose \bloss_t} &=
\EE{\sum_{t=1}^T \left(\tV_t - \bpi(\tSw)\right)\transpose \hbl_t}  
\le \frac{m\left(\log d+1\right)}{\eta}
\end{split}
 \]
 Next, we apply Lemma~\ref{lem:price} to obtain
\[
 \EE{(\bV_t - \tV_t)\transpose \hbl_t} \le \eta\, \EE{\left(\tV_{t-1}\transpose \hbl_{t}\right)^2} \le \eta m
\EE{\bV_t\transpose\bloss_t},
\] 
where we used that $\tbV_{t-1}\sim\bV_t$ and $\tbV_{t-1}\transpose\bloss_t\le m$.
Introducing the notation
 \[
 C_T = \sum_{t=1}^T \EE{\bV_t\transpose\bloss_t},
 \]
 we get by combining the above bounds that
\[
 C_T - L_T^* \le \frac{m\left(\log d+1\right)}{\eta} + \eta m C_T.
\]
After reordering, we get
\[
 C_T - L_T^* \le \frac{1}{1-m\eta}\pa{\frac{m\left(\log d+1\right)}{\eta} + \eta m L_T^*}.
\]
The bound follows from plugging in the choice of $\eta$ and observing that $1-m\eta\ge 1/2$ holds by the assumption of
the theorem.

\section{Proof details for Theorem~\ref{thm:restricted}}\label{app:restricted}
In the restricted information case, the term on the right hand side of the bound
of Lemma~\ref{lem:price} can be upperbounded as follows:
\[
\begin{split}
\EEcc{\pa{\tbV_{t-1}\transpose\hbl_t}^2}{\F_{t-1}}
&=
\EEcc{\sum_{j=1}^d\sum_{k=1}^d
\left(\wt{V}_{t-1,j}\hloss_{t,j}\right)\left(\wt{V}_{t-1,k}\hloss_{t,k}\right)}
{\F_{t-1}}
\\
&\leq
\EEcc{\sum_{j=1}^d\sum_{k=1}^d
\frac{N_{t,j}^2+N_{t,k}^2}{2}\left(\wt{V}_{t-1,j}A_{t,j}\loss_{t,j}
\right)\left(\wt{V}
_{ t-1 ,k}A_{t,k}\loss_{t,k}\right)} {\F_{t-1}}
\\
&\qquad\mbox{(using the definition of $\hbl_t$ and $2AB\le A^2 + B^2$)}
\\
&=
\EEcc{\sum_{j=1}^d\sum_{k=1}^d
N_{t,j}^2\left(\wt{V}_{t-1,j}A_{t,j}\loss_{t,j}
\right)\left(\wt { V }
_{ t-1 ,k}A_{t,k}\loss_{t,k}\right)} {\F_{t-1}}
\\
&\qquad\mbox{(by symmetry)}
\\
&\le
2\EEcc{\sum_{j=1}^d\frac{1}{a_j^2}\left(\wt{V}_{t-1,j}A_{t,j}\loss_{t,j}\right)\sum_{k=1}^d
\wt{V}_{t-1,k}\loss_{t,k}}{\F_{t-1}}
\\
&\qquad\mbox{(using $\EEc{N_{t,j}^2}{\F_{t-1}} = (2-a_j)/a_j^2 \leq
2/a_j^2$)}
\\
&=
2m\EEcc{\sum_{j=1}^d\frac{1}{a_j}\left(\wt{V}_{t-1,j}\loss_{t,j}\right)}{\F_{t-1}}
\\
&\qquad\mbox{(using $\bigl\|\tbV_t\bigr\|_1\le m$ and $\EEc{A_{t,j}}{\F_{t-1}} =
a_j$}
\\
&\le
2m\sum_{j=1}^d \EEcc{V_{t,j}}{j\in D_t, \F_{t-1}},
\end{split}
\]
where in the last line, we used that $\tbV_{t-1}$ is identically distributed as $\bV_t$.

\vfil

\section{Proof details for Theorem~\ref{thm:semibandit}}\label{app:semibandit}
In the semi-bandit case, the term on the right hand side of the bound of
Lemma~\ref{lem:price} can be upperbounded as
follows:
\[
\begin{split}
&\EEcc{\pa{\tbV_{t-1}\transpose\hbl_t}^2}{\F_{t-1}}
=
\EEcc{\sum_{j=1}^d\sum_{k=1}^d
\left(\wt{V}_{t-1,j}\hloss_{t,j}\right)\left(\wt{V}_{t-1,k}\hloss_{t,k}\right)}
{\F_{t-1}}
\\
&\quad\leq
\EEcc{\sum_{j=1}^d\sum_{k=1}^d
\frac{K_{t,j}^2+K_{t,k}^2}{2} \cdot \frac{N_{t,j}^2+N_{t,k}^2}{2}\left(\wt{V}_{t-1,j}V_{t,j}\loss_{t,j}
\right)\left(\wt{V}
_{ t-1 ,k}V_{t,k}\loss_{t,k}\right)} {\F_{t-1}}
\\
&\qquad\qquad\mbox{(using the definition of $\hbl_t$ and $2AB\le A^2 +
B^2$)}
\\
&\quad=
\EEcc{
\sum_{j=1}^d\sum_{k=1}^d
\frac{K_{t,j}^2N_{t,j}^2+K_{t,j}^2N_{t,k}^2+K_{t,k}^2N_{t,j}^2+K_{t,k}^2N_{t,k}^2}{4}\left(\wt{V}_{t-1,j}V_{t,j}\loss_{t
,j}
\right)\left(\wt{V}
_{ t-1 ,k}V_{t,k}\loss_{t,k}\right)} {\F_{t-1}}
\\
&\quad=
\EEcc{
\sum_{j=1}^d
\frac{K_{t,j}^2N_{t,j}^2}{2}\left(\wt{V}_{t-1,j}V_{t,j}\loss_{t
,j}
\right)\sum_{k=1}^d\left(\wt{V}
_{ t-1 ,k}V_{t,k}\loss_{t,k}\right)} {\F_{t-1}}
\\
&\quad\qquad
+
\frac 12 \cdot \EEcc{
\sum_{j=1}^d
K_{t,j}^2\left(\wt{V}_{t-1,j}V_{t,j}\loss_{t
,j}
\right)\sum_{k=1}^d N_{t,k}^2 \left(\wt{V}_{ t-1 ,k}V_{t,k}\loss_{t,k}\right)} {\F_{t-1}}
\\
&\qquad\le
\frac m2\cdot\EEcc{\EEcc{
\sum_{j=1}^d
M K_{t,j} N_{t,j}^2\left(\wt{V}_{t-1,j}V_{t,j}\loss_{t
,j}
\right)} {\F_{t-1},\Sw_t}}{\F_{t-1}}
\\
&\qquad\qquad
+
\frac 12\cdot \EEcc{\EEcc{
\sum_{j=1}^d
M K_{t,j}\left(\wt{V}_{t-1,j}V_{t,j}\loss_{t
,j}
\right)\sum_{k=1}^d N_{t,k}^2 \left(\wt{V}_{ t-1 ,k}A_{t,k}\loss_{t,k}\right)}
{\F_{t-1},\Sw_t}
}{\F_{t-1}}
\\
&\quad\qquad\mbox{(using $K_{t,j}\le M$ and $V_{t,k}\le A_{t,k}$ and 
$\bigl\|\tbV_t\bigr\|_1\le m$)}
\\
&\quad\le
\frac m2\cdot\EEcc{
\sum_{j=1}^d
M N_{t,j}^2\left(\wt{V}_{t-1,j}A_{t,j}\loss_{t
,j}
\right)}{\F_{t-1}}
\\
&\quad\qquad
+
\frac 12\cdot \EEcc{
\sum_{j=1}^d
M \left(\wt{V}_{t-1,j}A_{t,j}\loss_{t
,j}
\right)\sum_{k=1}^d N_{t,k}^2 \left(\wt{V}_{ t-1 ,k}A_{t,k}\loss_{t,k}\right)
}{\F_{t-1}}
\\
&\quad\qquad\mbox{(using $\EEcc{K_{t,j} V_{t,j}}{\F_{t-1},\Sw_t} \le
A_{t,j}$ by definition of $K_{t,j}$ and independence of $K_{t,j}$ and
$V_{t,j}$)}
\\
&\quad\le
2Mm\EEcc{
\sum_{j=1}^d \frac{1}{a_j^2}\left(\wt{V}_{t-1,j} A_{t,j}\loss_{t
,j}
\right)} {\F_{t-1}}
\\
&\quad\qquad\mbox{(using $\bigl\|\tbV_t\bigr\|_1\le m$ and
$\EEc{N_{t,j}^2}{\F_{t-1}} = (2-a_j)/a_j^2 \leq 2/a_j^2$)}
\\
&\quad=
2Mm\EEcc{\sum_{j=1}^d\frac{1}{a_j}\left(\wt{V}_{t-1,j}\loss_{t,j}\right)}{\F_{
t-1 }}
\\
&\quad\le
2Mm\sum_{j=1}^d \EEcc{V_{t,j}}{j\in D_t, \F_{t-1}},
\end{split}
\]
where in the last line, we used that $\tbV_{t-1}$ is identically distributed as $\bV_t$.

\vfil

\section{A lower bound for restricted feedback}\label{app:lowerbound}
Consider a sleeping experts problem with $d$ experts with loss sequence
$(\loss_t)_t$. In each round $t=1,2,\dots,T$ the
learner picks $I_t$. For simplicity, assume that $d$ is even and let $N
= d/2$. Let $\P$ be such that it assigns a probability of $1/N$ to each pair
$(2i-1,2i)$ of experts, that is, only two
experts are awake at each time. The regret of any learning algorithm in this problem can be written as
\[
 R_T = \max_\pi \EE{\sum_{t=1}^T \pa{\loss_{t,I_t} - \loss_{t,\pi(\Sw_t)}}} = 
 \max_\pi \EE{\sum_{j=1}^N \sum_{t=1}^T \II{2i\in\Sw_t} \pa{\loss_{t,I_t} - \loss_{t,\pi(\Sw_t)}}}.
\]
We now define $N$ full-information games $G_1,\dots,G_N$ with two experts each as follows: In game $G_i$, the number of
rounds is $T_i = \sum_{t=1}^T \II{2i\in\Sw_t}$, the decision of the learner in round $t$ is $J_t$ and the sequence
of loss functions is $(\loss_t(i)_t)$ so that the regret in game $G_i$ is
defined as
\[
 R_T(i) = \max_j \sum_{t=1}^{T_i} \left(\loss_{t,J_t}(i) - \loss_{t,j}(i)\right).
\]
It is well-known (e.g., \cite[Section~3.7]{CBLu06:Book}) that there exists a distribution of losses that guarantees
that $\EE{R_T(i)} \ge c\sqrt{T_i}$ for some constant $c>0$, no matter what algorithm the learner uses. The result
follows from observing that there exists a mapping between the full-information games $G_1,\dots,G_N$ and our original
problem such that
\[
 \EEcc{\sum_{j=1}^N \sum_{t=1}^T \II{2i\in\Sw_t} \pa{\loss_{t,I_t} - \loss_{t,\pi(\Sw_t)}}}{(\Sw_t)_{t=1}^T} = 
 \sum_{i=1}^N \EE{R_T(i)} \ge c \sum_{i=1}^N \sqrt{T_i}.
\]
That, is we get that
\[
 R_T \ge c \sum_{i=1}^N \EE{\sqrt{T_i}}.
\]
As $\lim_{T\ra\infty} \frac{T}{T_i} = N$ holds almost surely, we get that
\[
 \lim_{T\ra\infty} \frac{R_T}{\sqrt{TN}} \ge c,
\]
showing that there exist sleeping experts problems for any even $d$ where the
guarantees of Corollary~\ref{cor:rest_1}
cannot be improved asymptotically.

\end{document}